\documentclass[cis]{ipart_v1}

\firstpage{1}
\Vol{1}
\Issue{1}
\Year{2018}

\usepackage{tikz}
\usetikzlibrary{matrix,decorations.pathreplacing}

\newtheorem{theorem}{Theorem}

\newtheorem{definition}{Definition}

\newtheorem{corollary}[theorem]{Corollary}

\newtheorem{property}{Property}

\newtheorem{assumption}{Assumption}

\begin{document}

\title[Understanding SPL under Concave Conjugacy Theory]{Understanding Self-Paced Learning under Concave Conjugacy Theory}

\author[S.-Q~Liu, Z.-L.~Ma, D.-Y~Meng, K.-D.~Wang, and Y.~Zhang]{Shi-Qi Liu, Zi-Lu Ma, De-Yu Meng, Kai-Dong Wang, and
\\Yong Zhang\blfootnote{* Research supported
    by the NSFC projects under contracts 61661166011, 61721002 and Macau STDF under contract 003/2016/AFJ.}}

\begin{abstract}
By simulating the easy-to-hard learning manners of humans/animals, the learning regimes called curriculum learning~(CL) and self-paced learning~(SPL) have been recently investigated and invoked broad interests. However, the intrinsic mechanism for analyzing why such learning regimes can work has not been comprehensively investigated. To this issue, this paper proposes a concave conjugacy theory for looking into the insight of CL/SPL. Specifically, by using this theory, we prove the equivalence of the SPL regime and a latent concave objective, which is closely related to the known non-convex regularized penalty widely used in statistics and machine learning. Beyond the previous theory for explaining CL/SPL insights, this new theoretical framework on one hand facilitates two direct approaches for designing new SPL models for certain tasks, and on the other hand can help conduct the latent objective of self-paced curriculum learning, which is the advanced version of both CL/SPL and possess advantages of both learning regimes to a certain extent. This further facilitates a theoretical understanding for SPCL, instead of only CL/SPL as conventional. Under this theory, we attempt to attain intrinsic latent objectives of two curriculum forms, the partial order and group curriculums, which easily follow the theoretical understanding of the corresponding SPCL regimes.
\end{abstract}

\maketitle

\section{Introduction}

Since being raised recently, self-paced learning~(SPL)\cite{Kumar2010Self} and curriculum learning~(CL)\cite{Bengio2009Curriculum} have been grabbing attention in machine learning and artificial intelligence. Both learning paradigms are designed by simulating the learning principle of humans/animals, attempting to start learning from easier examples and gradually including more complex ones into the training process. The CL regime~\cite{Bengio2009Curriculum,CLApp1,CLAPP3} was formerly designed by setting a series of learning curriculums for ranking samples from easy to hard manually, and the SPL methodology \cite{Kumar2010Self} has then been latterly proposed to make this easy-to-hard learning manner automatically implementable by imposing a regularization term into a general learning object, which enables the learning machine to objectively evaluate the ``easiness" of a sample and automatically learn the object in an adaptive way. This learning paradigm has been empirically verified to be helpful on alleviating the local-minimum issue for a non-convex optimization problem \cite{Zhao2015Self}, while later on more comprehensively to be verified to be capable of making the learning method more robust to heavy noises/outliers \cite{Meng2015What}. Recently, such a new learning regime has been applied to many practical problems, such as multimedia event detection~\cite{Jiang2014Easy}, neural network training~\cite{avramova2015curriculum}, matrix factorization~\cite{Zhao2015Self}, multi-view clustering~\cite{Xu2015Multi}, multi-task learning~\cite{li2016self}, boosting classification~\cite{PiSelf}, object tracking~\cite{Supancic2013}, person re-identification~\cite{reid2013}, face identification~\cite{LinLiang2018}, object segmentation~\cite{Objectsegmentation}, and  some related mechanisms have been applied to weakly supervised learning~\cite{liang2015towards},\cite{wei2017stc},
\cite{liang2017learning}. Furthermore, an intrinsic advanced version of CL/SPL, called self-paced curriculum learning (SPCL)~\cite{Jiang2015Self}, has been designed, which tends to inherit advantages of both SPL and CL and to have a broader application~\cite{jiang2015bridging}. Besides, many variations of SPL realization schemes have also been constructed, like self-paced reranking~\cite{Jiang2014Easy}, self-paced multiple instance learning~\cite{SPMIL,SPMIL-PAMI}, self-paced learning with diversity~\cite{SPLD}, multi-objective self-paced learning~\cite{multiobjective2016}, self-paced co-training~\cite{cotraining2017} and etc.

For understanding the theoretical insights of the working mechanism underlying the CL/SPL strategy, some beneficial investigations have been made. Meng et al~\cite{Meng2015What} proved that the alternative search algorithm generally used to solve the self-paced learning problem is equivalent to a majorization minimization algorithm implemented on a latent SPL object function, which is closely related to the non-convex penalty used in statistics and machine learning~\cite{Meng2015What}. This follows a natural explanation for the intrinsic robustness of CL/SPL. Recently, they have further proved that SPL scheme converges to a critical point of the latent objective \cite{SPLConverge}. Afterwards, Fan et al.~\cite{Fan2016Self} explored an implicit regularization perspective of self-paced learning, which also conducts similar robust understandings for this learning regime. Recently, Li et al.~\cite{Li2017Self} proposed a general way to find the desired self-paced functions, which is beneficial for constructing more variations of SPL forms in practice.

However, these investigations explore the SPL theory mainly through exploring the equivalence of the alternative search algorithm on the SPL objectives and other algorithms implemented on some latent objective functions, while not on the SPL objective function, as well as its self-paced regularizer, itself. This makes the theory not sufficiently insightful to the problem. For example, the intrinsic relationships between self-paced regularizers and the weighting scheme to measure the importance of training samples in a SPL model is generally implicit, and hard to be intuitively explained. Besides, after adding curriculum constraint in SPL regime to form a SPCL model, current theories cannot attain the latent function like under general SPL framework. The rationality of SPCL thus still rests on the intuitive level.

To alleviate these issues, this study mainly makes the following contributions: Firstly, we establish a systematic theoretical framework under concave conjugacy theory for understanding the CL/SPL/SPCL insights. We find that the concave conjugacy theory surprisingly tallies with the requirements of the SPL model. And under this framework, the relationship among self-paced regularizer, latent SPL object function and sample weights can be clarified in a theoretically sound manner. Besides, by using this theory, the redundancy of the original SPL axiom can be removed and simplified, and the influence of the age parameter can be interpreted. Secondly, we can render a general approach for designing the SPL regime by using this theory. Furthermore, one can easily embed the required prior knowledge directly to the sample weights under this framework to make it properly used in specific applications. Thirdly, the latent objective of SPCL can be obtained under this theory. We especially discuss the form of the latent objective functions of SPCL under the partial order and group curriculums. This theory is thus meaningful for providing generalizable explanation for more general CL/SPL variations.

The paper is organized as follows. Section 2 introduces the necessary concepts and theories on concave conjugacy. Section 3 proposes the concave conjugacy theory for understanding CL/SPL. Section 4 presents two general approaches for designing a specific SPL model. Section 5 provides the theoretical understanding for SPCL under this new theory, and discusses the latent objectives of two specific curriculums.

\section{Related contents on concave conjugacy}

In the following we use the bolded lower letter to denote a vector, and the non-bolded lower letter to denote a scaler. For $\mathbf{v}$ and $u$, denote $(\mathbf{v},u)$ as a vector in $\mathcal{R}^{n+1}$ by arranging $u$ after the last position of $\mathbf{v}$. The inequality $\mathbf{v}\succeq \mathbf{u}$ means that satisfies $v_i \ge u_i$ for $i=1,\cdots,n$; $ \langle \mathbf{v},\mathbf{l} \rangle =\mathbf{v}^{T}\mathbf{l}$ denotes the inner products of $\mathbf{v}$ and $\mathbf{l}$. For a concave function, we assume that it takes $-\infty$ out of its domain; for a convex function, we assume that it takes $+\infty$ out of its domain. Before giving more related concepts, we first presents the following definition.

 \begin{definition}[Increasing Function] A multivariate function $f(\mathbf{v})$ is increasing if  $f(\mathbf{v})\ge f(\mathbf{u})$ for all $\mathbf{v}\succeq \mathbf{u}$ lying in its domain denoted by $dom \ f$.
 \end{definition}

\subsection{Conjugate}

We first present some necessary concepts and their related properties on the conjugate theory.

\begin{definition}[Hypograph]The \textbf{hypograph} associated with the function $g :\mathcal{R}^n\longrightarrow \mathcal{R}$ is the set of points lying on or below its graph:
$$hyp\ g=\{(\mathbf{v},u):\mathbf{v}\in \mathcal{R}^n, u\in \mathcal{R}, u\le g(\mathbf{v})\}\subset \mathcal{R}^{n+1}.$$
\end{definition}

\begin{property}[Hypograph Correspondence\cite{Rockafellar1970}]
The function $g(\cdot)$ and its hypograph satisfy the following correspondence:
$$g(\mathbf{v})=\sup \limits_{(\mathbf{v},u)\in hyp \ g}u.$$
\end{property}

\begin{property}[Concave function]
  $g(\cdot)$ is a concave function  if and only if  $hyp \ g$  is a convex set.
\end{property}

\begin{definition}[Closure of Function]
 The closure of the function $g(\cdot)$ is a function generated by the closure of its hypograph: $$cl\ g=\sup \limits_{(\mathbf{v},u)\in cl \ (hyp \ g)}u.$$ It yields $$hyp\ (cl\ g)=cl\ (hyp\ g).$$
\end{definition}

\begin{definition}[Concave Conjugate]The concave conjugate of a function $g(\cdot)$ is defined as follows: $$g^*(\mathbf{l})=\inf \limits_{\mathbf{v}\in \mathcal{R}^n}\{\langle \mathbf{v},\mathbf{l} \rangle-g(v)\}.$$
\end{definition}

\begin{property}[Relation of Concave Conjugate and Convex Conjugate\cite{Rockafellar1970}\label{Relation of Concave Conjugate and Convex Conjugate}]
  For a convex function $f(\mathbf{v})=-g(\mathbf{v})$, it holds that: $$g^*(\mathbf{l})=-f^*(-\mathbf{l})$$ where $f^*(\mathbf{l})$ is the convex conjugate of $f(\mathbf{v})$ defined as: $$f^*(\mathbf{l})=\sup \limits_{\mathbf{v}\in \mathcal{R}^n}\{\langle \mathbf{v},\mathbf{l} \rangle-f(\mathbf{v})\}.$$
\end{property}

For notation convenience, in the follows we also use conjugate to represent concave conjugate.

\begin{definition}[Proper Function]
  A concave function $g(\cdot)$ is proper if it takes value on $[-\infty,+\infty)$ and there is at least one $\mathbf{v}$ such that $g(\mathbf{v})>-\infty$.
\end{definition}

Following the proof given by W.Fenchel~\cite{fenchel1949conjugate} regarding the property of the conjugate convex function, one can easily prove that if $g(\mathbf{v})$ is proper, then $g^*(v)$ is a closed concave function. The concave conjugacy inherits the following duality properties of convex conjugacy as well.

\begin{property}[Duality\cite{Rockafellar1970}]
  If $g(\cdot)$ is a upper semi-continuous, concave and proper function, $$g^{**}(\mathbf{v})=g(\mathbf{v})$$ i.e. $$g(\mathbf{v})=\inf \limits_{l\in \mathcal{R}^n}\{\langle \mathbf{v},\mathbf{l}, \rangle-g^*(\mathbf{l})\}.$$
\end{property}

It can be observed that the concave conjugate presents a one-to-one correspondence for all closed proper concave functions defined on $\mathcal{R}^n$.

\subsection{Additive properties}
The additive properties of concave conjugacy are also required to prove the related theory for SPL. We thus introduce the following necessary definitions and properties.

\begin{definition}[Sup-Convolution] The sup-convolution of functions $f(\cdot)$ and $g(\cdot)$ is defined as: $$f\oplus g(\mathbf{v})=\sup \limits_{\mathbf{v}_1+\mathbf{v}_2=\mathbf{v}}\{f(\mathbf{v}_1)+g(\mathbf{v}_2)\}$$
\end{definition}

The sup-convolution has the following properties:

\begin{property}[Increasing and Concave Preserving]
Let $h=f\oplus g$, and then
  \begin{itemize}
  \item  if $f(\cdot)$ and $g(\cdot)$ are increasing function, so is $h$;
  \item if $f(\cdot)$ and $g(\cdot)$ are concave function, so is $h$.
\end{itemize}
\end{property}

The relationship between the sup-convolution and the concave conjugate can be well illustrated by the following result.

\begin{property}[Additive Property\label{Additive Property}] Let $g_1(\cdot),\cdots,g_m(\cdot)$ be proper concave functions defined on $\mathcal{R}^n$. Then we have:
$$(g_1\oplus\cdots\oplus g_m)^*=g_1^*+\cdots+g_n^*,$$ $$(cl\ g_1+\cdots+cl\ g_m)^*=cl(g_1^*\oplus\cdots\oplus g_m^*).$$
If the relative interior of $(dom \ g_i),i=1,\cdots,m$ have a point in common, the closure operation can be omitted from the above second formula, and $$(g_1+\cdots+g_m)^*=g_1^*\oplus\cdots\oplus g_m^*,$$
$$(g_1+\cdots+g_m)^*(\mathbf{l})=\sup \limits_{\mathbf{l}^1+\cdots+\mathbf{l}^m=\mathbf{l}}\{g_1^*(\mathbf{l}^1)+\cdots+g_m^*(\mathbf{l}^m)\},$$ where for each $\mathbf{l}$ the supremum is attained.
\end{property}
The proof of this property can be referred to in \cite{Rockafellar1970}.

\subsection{Differential theory}
The differential theory regarding the concave conjugate plays an important role in our SPL theory. Some necessary definitions and properties are thus introduced as follows.

\begin{definition}[Subgradient]A vector $\mathbf{l}$ is a subgradient of a concave function $g(\cdot)$ at $\mathbf{v}$ if $$g(\mathbf{z})\le g(\mathbf{l})+\langle \mathbf{l},\mathbf{z}-\mathbf{l}\rangle, \forall \mathbf{z}\in\mathcal{R}^n.$$
The set of all subgradients of $g(\cdot)$ at $\mathbf{v}$ is called the subdifferential of $g(\cdot)$ at $\mathbf{v}$ and is denoted by $\partial g(\mathbf{v}) $.
\end{definition}

Correspondingly, the subgradient $\mathbf{l}$ of a convex function $f(\mathbf{v})=-g(\mathbf{v})$ at $\mathbf{v}$ if $$f(\mathbf{z})\ge f(\mathbf{v})+\langle l,\mathbf{z}-\mathbf{v}\rangle, \forall \mathbf{z}\in\mathcal{R}^n.$$
The set of all subgradients of $f(\cdot)$ at $\mathbf{v}$ is called the subdifferential of $f(\cdot)$ at $\mathbf{v}$ and is denoted by $\partial f(\mathbf{v})$.

The above subdifferentials of $f(\cdot)$ and $g(\cdot)$ have the following relation $$\partial g(v)=-\partial f(v).$$.

\begin{property}[Duality of Subdifferential~\cite{Rockafellar1970}]\label{DoS}For any closed proper concave function $g(\cdot)$ and any vector $\mathbf{v}$, the following conditions on a vector $\mathbf{l}$ are equivalent to each other:
\begin{itemize}
  \item $\mathbf{l}\in \partial g(\mathbf{v})$;
  \item $\langle \mathbf{z},\mathbf{l} \rangle -g(\mathbf{z})$ achieves its infmum in $\mathbf{z}$ at $\mathbf{z}=\mathbf{v}$;
  \item $g(\mathbf{v})+g^*(\mathbf{l})=\langle \mathbf{z},\mathbf{l} \rangle$;
  \item $\mathbf{v}\in\partial g^*(\mathbf{l})$;
  \item $\langle \mathbf{v},\mathbf{z} \rangle -g(\mathbf{z})$ achieves its infmum in $\mathbf{z}$ at $\mathbf{z}=\mathbf{l}$.
\end{itemize}
\end{property}

\begin{property}[Structure of Subdifferential~\cite{Rockafellar1970}]\label{SoS} Let $g(\cdot)$ be a closed proper concave function such that $dom \ g$ has a non-empty interior. Then $$\partial g(\mathbf{x})=cl(conv S(\mathbf{x}))+K(\mathbf{x}) \ \forall \mathbf{x}\in\mathcal{R}^n,$$ where $K(\mathbf{x})=\{\mathbf{x}^*|\langle \mathbf{y}-\mathbf{x},\mathbf{x}^*\rangle \ge0 \ \forall \mathbf{y}\in dom\ g\}$ is the normal cone to $dom\ g$ at $\mathbf{x}$ and $S(\mathbf{x})$ is the set of all limits of sequences $(\nabla g(x_1),\nabla g(x_2),\cdots)$ such that $g(\cdot)$ is differentiable at $\mathbf{x}_i$ and $\mathbf{x}_i$ converges to $\mathbf{x}$.\end{property}

\begin{theorem}[Duality of essential strictly convex and essentially smooth\label{Duality of essential strictly concave and essentially smooth}\cite{Rockafellar1970}] A closed proper convex function is essential strictly convex if and only if its conjugate is essential smooth.
\end{theorem}

\begin{corollary}\label{Strictly Convexity implies Differentiability} If $f(\cdot)$ is a closed strictly convex function with bounded domain, then $f^*(\cdot)$ is a closed differentiable function on the whole space.
\end{corollary}
\begin{proof}
  Since $f(\cdot)$ is with bounded domain, we know $f(\cdot)$ is co-finite. And then we have that $f^*(\cdot)$ is defined on whole space~\cite{Rockafellar1970}.

  Furthermore, since $f(\cdot)$ is strictly convex, we can deduce that it is essential strictly convex~\cite{Rockafellar1970}.
  According to theorem~\ref{Duality of essential strictly concave and essentially smooth}, $f^*(\cdot)$ is essential smooth on the whole space, meaning that $f^*(\cdot)$ is differentiable on the whole space~\cite{Rockafellar1970}.
\end{proof}

\subsection{Indicator function}
The following theory illustrates that a restriction imposed on feasible region can be viewed as the addition of an indicator function of the restricted feasible region to the objective function.

\begin{definition}[Indicator Function]
The indicator function of a convex set $C\subset \mathcal{R}^n$ is defined by: $$\delta(\mathbf{v}|C)=\left\{
\begin{aligned}
0 &  & \mathbf{v}\in C,\\
-\infty &  & \mathbf{v}\notin C. \\
\end{aligned}
\right
.$$
\end{definition}
The closure of  $\delta(\mathbf{v}|C)$ satisfies $cl\ \delta(\mathbf{v}|C)= \delta(\mathbf{v}|cl\ C)$.

\begin{definition}We call the conjugate of $\delta(\mathbf{v}|C)$ the support function of $C$ : $$\delta^*(\mathbf{l}|C)=\inf \limits_{\mathbf{v}\in C}{\langle \mathbf{v},\mathbf{l} \rangle}.$$
\end{definition}

Based on the above definitions of indictor function and support function, the concave conjugate with constraint can be interpreted in a new way. Specifically, suppose $g(\cdot)$ is a upper semi-continuous, proper, concave function, $\Psi$ is a closed convex set and the relative interior of $dom\ g$ and $\Psi$ have at lease a point in common. Then we have
$$\inf \limits_{\mathbf{v}\in \Psi}\{\langle \mathbf{v},\mathbf{l} \rangle-g(\mathbf{v})\}=\inf \limits_{\mathbf{v}\in \mathcal{R}^n}\{\langle \mathbf{v},\mathbf{l} \rangle-g(\mathbf{v})-\delta(\mathbf{v}|\Psi)\}$$$$=(g(\mathbf{v} )+\delta(\mathbf{v}|\Psi))^*=g^*\oplus \delta^*(\mathbf{l} |\Psi).$$
This implies that a concave conjugate with domain constraint can be understood as the addition of two conjugate terms. This will help a lot to deduce the related theory on explaining SPCL. Details will be shown in Section 4.

\begin{theorem}[Monotone Conjugate]\label{Monotone Conjugate} If $g(\mathbf{v})$ is a function defined on a closed set $\Psi\subset \mathcal{R}_+^n$, then  $$g^*(l)=\inf \limits_{\mathbf{v}\in \Psi}\{\langle \mathbf{v},\mathbf{l} \rangle-g(\mathbf{v})\}$$ is increasing on $\mathcal{R}^n$.
\end{theorem}
The proof of this theorem can be seen in Appendix~A.

\section{Concave conjugate theory for SPL  \label{sec2}}

\subsection{SPL Regime}
We first give a short review to the generally used SPL regime.

For a given data set $D=\{\mathbf{z}_i\}_{i=1}^n$, where $z_i=(\mathbf{x}_i,y_i)$ is a training sample with a datum and its corresponding label, SPL uses the following model for learning \cite{Jiang2014Easy,Zhao2015Self}:
  \begin{equation}\label{self-paced learning model}
    \inf_{f\in \mathcal{F},\mathbf{v}\in [0,1]^n}E(f,\mathbf{v};\lambda)=\inf_{f_w\in \mathcal{F},\mathbf{v}\in [0,1]^n}\sum_{i=1}^{n}v_i L(f,\mathbf{z}_i)+R_{SP}(\mathbf{v},\lambda)+R_\mathcal{F}(f),
  \end{equation}
   where $\mathbf{v}=(v_1,v_2,\cdots,v_n)=(Weight(\mathbf{z}_1),\cdots,Weight(\mathbf{z}_n))$ represent the vector of weights imposed on all training samples, $R_{SP}(\mathbf{v},\lambda)$ is called self-paced regularizer which encodes the learning procedures following the principle from easy to hard, $R_\mathcal{F}(f)$ is the general regularizer for the model parameters to alleviate the overfitting problem, and $\lambda$ is a parameter that controls the learning pace and guarantees the easy-to-complex learning procedure. By gradually increasing the age parameter, more samples can be automatically included with higher weights into training in a purely self-paced way. $f$ is the decision function for the task, like a classifier or a regressor, $L(\cdot,\cdot)$ is the loss function (the function $f$ is generally parameterized by parameters $\mathbf{w}$ and $L$ is then the function with respect to $\mathbf{w}$ and $\mathbf{z}$). Let $\mathbf{l}$ denote the loss vector $(L(f,\mathbf{z}_1),\cdots,L(f,\mathbf{z}_n))^T$. This leads to a brief expression for the model:
    $$\inf_{\mathbf{w}\in \mathcal{W},\mathbf{v}\in [0,1]^n}\langle \mathbf{v},\mathbf{l}\rangle+R_{SP}(\mathbf{v},\lambda)+R_\mathcal{W}(\mathbf{w}).$$

  A common way to solve the SPL model is to alternatively optimize the target function $f$ and the weight vector $\mathbf{v}$ as follows:
\begin{itemize}
  \item Optimize $f$:
  \begin{equation}
  f^{k}=\inf_{f\in \mathcal{F}}  \langle \mathbf{v}^{k-1},\mathbf{l}(f)\rangle+R_\mathcal{F}(f).
  \end{equation}

  \item Optimize $\mathbf{v}$:
  \begin{equation}
  \mathbf{v}^{k}=\inf_{\mathbf{v}\in [0,1]^n} \langle \mathbf{v},\mathbf{l}(f^{k})\rangle+R_{SP}(\mathbf{v},\lambda).
  \end{equation}
  \end{itemize}

The SP-regularizer should satisfy necessary conditions to guarantee an expected easy-to-hard learning manner \cite{Jiang2014Easy,Zhao2015Self}:

  \begin{definition}[SP-regularizer\label{SP-regularizer}]$R_{SP}(\mathbf{v},\lambda)$ is called a SP-regularizer, if
  \begin{itemize}
  \item $R_{SP}(\mathbf{v},\lambda)$ is convex with respect to $\mathbf{v}\in [0,1]^n$;

  \item $v_i(\lambda,l_i)$ decrease with respect to $l$, and it holds that $\forall i \in \{1,2,\cdots,n\}$, $v_i(\lambda,l_i)\le1$ and $\lim \limits_{l_i\rightarrow +\infty}v_i(\lambda,l_i)=0$;

  \item $v_i(\lambda,l_i)$ increase with respect to $\lambda$, and it holds that $\forall i \in \{1,2,\cdots,n\}$, $v_i(\lambda,l_i)\le1$ and $\lim \limits_{\lambda\rightarrow 0}v_i(\lambda,l_i)=0(l_i>0)$,
  \end{itemize}

where $\mathbf{v}(\lambda,\mathbf{l})=\arg \inf \limits_{\mathbf{v}\in [0,1]^n}\{\langle \mathbf{v},\mathbf{l} \rangle+R_{SP}(\mathbf{v},\lambda)\}$.
\end{definition}

By using such defined SP-regularizer, SPL can conduct the learning manner that imposes larger weights on easier samples while smaller on harder ones, and gradually increases the sample weights with the age parameter increasing.

\subsection{Conjugate theory of SP-regularizer}

We can prove the following conjugate result on a SP-regularizer $R_{SP}(\mathbf{v},\lambda)$ as follows:

\begin{theorem}[Conjugate Equivalence\label{Conjugate Equivalence}]  For arbitrary function $R_{SP}(\mathbf{v})$ satisfying $dom_\mathbf{v}\ R_{SP}(\mathbf{v}) \subset[0,1]^n$, let $g(\mathbf{v})=-R_{SP}(\mathbf{v})$, and then
$$\inf \limits_{\mathbf{v}\in [0,1]^n}\{\langle \mathbf{v},\mathbf{l} \rangle+R_{SP}(\mathbf{v})\}=\inf \limits_{\mathbf{v}\in [0,1]^n}\{\langle \mathbf{v},\mathbf{l} \rangle-g^{**}(\mathbf{v})\}=\inf \limits_{\mathbf{v}\in [0,1]^n}\{\langle \mathbf{v},\mathbf{l} \rangle+\overline{{R_{SP}}}(\mathbf{v})\},$$
where $\overline{{R_{SP}}}(\mathbf{v})=-g^{**}(\mathbf{v}).$
\end{theorem}
The proof is provided in Appendix~B.

From the above theorem, it can be found that there are redundancy in the definition of SP-regularizer, which can be simplified as follows:

\begin{theorem}[SP-regularizer Simplification\label{SP-regularizer Simplification}] If $R_{SP}(\mathbf{v},\lambda)$ satisfies
\begin{itemize}
  \item $R_{SP}(\mathbf{v},\lambda)$ is strictly convex in $\mathbf{v}$;
  \item $R_{SP}(\mathbf{v},\lambda)$ is lower semi-continuous in $\mathbf{v}$;
  \item $dom_\mathbf{v} \ R_{SP}(\mathbf{v},\lambda)\subset[0,1]^n$ and $\textbf{0},\textbf{1}\in cl (dom_\mathbf{v} \ R_{SP}(\mathbf{v},\lambda))$,
\end{itemize}
then it holds that $\forall i \in \{1,2,\cdots,n\}$:
\begin{itemize}
 \item $v_i(\lambda,l_i)$ decrease with respect to $l_i$; $v_i(\lambda,l_i)\le1$ ; $\lim \limits_{l_i\rightarrow +\infty}v_i(\lambda,l_i)=0.$

 \item If $R_{SP}(\mathbf{v},\lambda)=\lambda R_{SP}(\mathbf{v})$ where $R_{SP}(\mathbf{v})$ satisfy the above condition in $\mathbf{v}$, then $\forall i \in \{1,2,\cdots,n\}$,
$v_i(\lambda,l_i)$ increases with respect to $\lambda$, $v_i(\lambda,l_i)\le1$, $\lim \limits_{\lambda\rightarrow 0}v_i(\lambda,l_i)=0 \ (l_i>0)$,
\end{itemize}where $\mathbf{v}(\lambda,\mathbf{l})=\arg \inf \limits_{\mathbf{v}\in [0,1]^n}\{\langle \mathbf{v},\mathbf{l} \rangle+R_{SP}(\mathbf{v},\lambda)\}$.
\end{theorem}

The proof is presented in Appendix~C.

This theorem shows that the conditions in $\mathbf{l}$ can be implied by the conditions being directly imposed on the SP-regularizer.
According to simplification theorem, determining one easily handled representative in the equivalence class, the following assumption gives weaker conditions for a SP-regularizer.

\begin{definition}[SP-regularizer simplification]\label{assumption for SP-regularizer} $R_{SP}(\mathbf{v},\lambda)$ is called a self-paced regularizer with simplified conditions if:
\begin{enumerate}
  \item $R_{SP}(\mathbf{v},\lambda)$ is convex in $\mathbf{v}$;
  \item $R_{SP}(\mathbf{v},\lambda)$ is lower semi-continuous in $\mathbf{v}$;
  \item $dom_\mathbf{v} \ R_{SP}(\mathbf{v},\lambda)\subset[0,1]^n$ and $\textbf{0},\textbf{1}\in cl (dom_\mathbf{v} \ R_{SP}(\mathbf{v},\lambda))$.
\end{enumerate}
\end{definition}

\subsection{Model Equivalence}
Based on the concave conjugacy of SPL, its equivalent model can be derived as follows.
For convenience, let $g_\lambda(\mathbf{v})=-R_{SP}(\mathbf{v},\lambda)$, and then it holds that:
  \begin{eqnarray*}
    &&\inf_{f\in \mathcal{F},\mathbf{v}\in [0,1]^n}E(f,\mathbf{v};\lambda) \\
     &\Longleftrightarrow& \inf_{f\in \mathcal{F}}R_\mathcal{F}(f)+\inf_{\mathbf{v}\in [0,1]^n}\sum_{i=1}^{n}v_i L(f,z_i)+R_{SP}(\mathbf{v},\lambda) \\
     &\Longleftrightarrow&\inf_{f\in \mathcal{F}}g_\lambda^*(\mathbf{l}(f))+R_\mathcal{F}(f)\Longleftrightarrow\inf_{f\in \mathcal{F}}F_\lambda(\mathbf{l}(f))+R_\mathcal{F}(f)\label{eq spl model}
  \end{eqnarray*}
where $F_\lambda(\mathbf{l})=g_\lambda^*(\mathbf{l})$. According to the property of the concave conjugate, $F_\lambda(\mathbf{l})$ is a proper closed concave function. Through this analysis, we can try to get more insights of SPL.

\subsubsection{Latent SPL objective}

Mostly, we can separate a SPL optimization model to multiple $1$ dimension sub-problems: $$\inf \limits_{f\in\mathcal{F},\mathbf{v}\in [0,1]^n} E(f,\mathbf{v};\lambda)=\inf \limits_{f\in\mathcal{F},\mathbf{v}\in [0,1]^n}\{\sum_{i=1}^n (v_i l_i+R_{SPi}(v_i,\lambda))+R_\mathcal{F}(f)\}.$$

Then, the optimization on $\mathbf{v}$ can be reformulated as solving the following multiple subproblems on each of its component $v_i$: $$\inf \limits_{v\in [0,1]} E(w,v;\lambda)=\inf \limits_{v\in [0,1]}\{v l+R_{SP}(v,\lambda)\}.$$
We denote
$$v(\lambda,l)=\arg \inf \limits_{v\in [0,1]}\{\langle v,l \rangle+f(v,\lambda)\}$$

In \cite{Meng2015What}, it is proved that the alternative search algorithm on the SPL objective is equivalent to the MM algorithm implemented on a latent objective $$\int_0^l v(\lambda,j)\,dj$$ on l. We can get the similar result under concave conjugate theory as follows.

\begin{theorem}[Model Equivalence\label{Model Equivalence}] If $R_{SP}(v,\lambda)$ satisfy the simplified conditions of SPL as defined in \ref{assumption for SP-regularizer} and be strictly convex, then the latent SPL objective can be written as:
$$F_\lambda(l)=\int_0^l v(\lambda,j)\,dj+C(\lambda),$$
  where $C(\lambda)$ is a function in $\lambda$.
\end{theorem}
The proof is listed in Appendix~D.

\subsection{Relations}

In the following theorem, we want to make the relations among the SP-regularizer $R_{SP}(v,\lambda)$, latent objective $F_\lambda(l)$, and the weight function $v(\lambda,l)$ clear.

\begin{theorem}\label{Relations} If $R_{SP}(v,\lambda)$ satisfy the simplified conditions of SPL, then we have:
\begin{eqnarray*}
  l_\lambda(v) &=&\partial_v(-R_{SP}(v,\lambda)), \\
  v(\lambda,l) &=& l_\lambda^{-1}(l), \\
  v(\lambda,l) &=& \partial F_\lambda(l), \\
  F_\lambda(l) &=&  \langle v(\lambda,l),l\rangle+R_{SP}(v(\lambda,l),\lambda),  \\
  R_{SP}(v,\lambda) &=&  \langle v,l_\lambda(v)\rangle-R_{SP}(v,\lambda)(l_\lambda(v)).
  \end{eqnarray*}
  Furthermore, if $R_{SP}(v,\lambda)$ and $F_\lambda(l)$ is strictly convex in $v $ and $l$, respectively and we can further obtain that
  \begin{eqnarray*}
  F_\lambda(l) &=& \int_0^l v(\lambda,j)\,dj+C(\lambda),\\
  R_{SP}(v,\lambda) &=& -\int_0^v l_\lambda(j)\,dj+C(\lambda).
  \end{eqnarray*}
\end{theorem}

The theorem is directly got from the Duality of Subdifferential Properties \ref{DoS} and the latter two inequalities can be obtained based on Theorem \ref{Model Equivalence}.

According to Theorem \ref{Relations}, one can easily derive the weight function from the SP-regularizer through the differential and inverse step, which is empirically more convenient than through the arg-minimization analysis. We then discuss on how to specify the age parameter in the model.

\subsection{On age parameter\label{Influence of age parameter}}

An easy way to construct a SP-regularizer is first to generate a regularizer, denoted by $R_{SP}(\mathbf{v})$, satisfying the simplified conditions of SPL, and then use the SP-regularizer as $\lambda R_{SP}(\mathbf{v})$. The reason why it works can be interpreted as follows:

Let $g(\mathbf{v})=-R_{SP}(\mathbf{v})$ and let the concave conjugate of $g^*(\mathbf{l})=F(\mathbf{l})$. Then we have:
 $$F_\lambda(\mathbf{l})=(\lambda g(\mathbf{v}))^*=\inf \limits_{\mathbf{v}\in [0,1]^n}\{\langle \mathbf{v},\mathbf{l} \rangle-\lambda g(\mathbf{v})\}$$$$=\lambda \inf \limits_{\mathbf{v}\in [0,1]^n}\{<\mathbf{v},\lambda^{-1}\mathbf{l}>- g(\mathbf{v})\}=\lambda F(\lambda^{-1}\mathbf{l}).$$

For simplicity,  we assume $g(\mathbf{v})$ is strictly concave. As a result, $F(\mathbf{l})$ is differentiable and the original $\mathbf{v}(\mathbf{l})=\nabla F(\mathbf{l})$, and then we have: $$\mathbf{v}(\lambda,\mathbf{l})=\nabla_\mathbf{l} F_\lambda(\mathbf{l})=\lambda \nabla_\mathbf{l} F(\lambda^{-1}\mathbf{l})=\mathbf{v}(\lambda^{-1}\mathbf{l}).$$

Thus, $v_i(\lambda,l_i)$ increase with respect to $\lambda$, and it holds that $\forall i \in \{1,2,\cdots,n\}$, $\lim \limits_{\lambda\rightarrow 0}v_i(\lambda,l_i)=\lim \limits_{\lambda\rightarrow 0}v_i(\lambda^{-1}l_i)=0 $ and $\lim \limits_{\lambda\rightarrow +\infty}v_i(\lambda,l_i)=\lim \limits_{\lambda\rightarrow +\infty}v_i(\lambda^{-1}l_i)\le1$.

Besides, since  $v(\lambda,l)=v(\lambda^{-1}l)$, the change of the $\lambda$ stretches the shape of $v(\lambda,l)$. In particular, if the $v(l)$ is with threshold, then $v(\lambda,l)$ shifts the threshold through $\lambda$ which reflexes the change of decision boundary regarding learning or not.

Then we give a discussion on how to specify a proper age parameter in the learning process.

Generally the SP-regularizer has the data screening properties, that is, there exists some $\lambda^*$ such that $v(l\ge \lambda^*)=0$. One can use two ways for specifying the age parameter. The first is suggested by \cite{Kumar2010Self}: first to choose a $\lambda$ such that around half of example are used with positive weight, and then gradually increase the $\lambda$ to include more samples into training.
Another strategy is suggested in \cite{Jiang2014Easy}: first calculate
the loss of each example, and choose a age parameter such that a portion of samples with smaller loss is with positive weights and the other with zero weights; and then increase the portion number to implicitly increase the age parameter. Also some other variations~\cite{avramova2015curriculum} have also been discussed and can be considered in application.

\section{Two methods for designing a SPL regime\label{sec3}}

By utilizing the aforementioned theoretical results, we can construct two methods for designing a general SPL regime in practice.

We call the first method as the vFlR$\lambda$ method. The progress
for one dimension sub-problem is provided as follows:
\begin{enumerate}
  \item Design $v(l)$ satisfying
$v(l)$ decrease with respect to $l$ and
 $$\lim \limits_{l\rightarrow 0}v(l)=1 \ \ \lim \limits_{l\rightarrow +\infty}v(l)=0;$$
  \item $ F(l)=\int_0^l v(j)\,dj$;
  \item $l(v)=v^{-1}(v)$;
  \item $R_{SP}(v)=-\langle v,l(v)\rangle+F(l(v))$;
  \item  $R_{SP}(v,\lambda)=\lambda R_{SP}(v)$;
          $F_\lambda(l)=\lambda F(\lambda^{-1}l)$;
          $v(\lambda,l)=v(\lambda^{-1}l)$.
\end{enumerate}

If $F(l)$ is given then $v(l)=\partial F(l)$ and the other steps are the same.

We can then provide an example for designing SPL by using this method.

\begin{enumerate}
  \item $v(l)=(1-l)_{[0,1]}$;
  \item $ F(l)=\int_0^l v(j)\,dj=\min(l-\frac{l^2}{2},\frac{1}{2})$;
  \item $l(v)=v^{-1}(v)=\left\{\begin{aligned}
1-v &  & v\in (0,1], \\
[1,+\infty) &  & v=0;
\end{aligned}
\right.$
  \item $R_{SP}(v)=-\langle v,l(v)\rangle+F(l(v))$, whose component is computed by $\frac{(1-v)^2}{2}$;
  \item $R_{SP}(v,\lambda)=\lambda R_{SP}(v) =\frac{\lambda(1-v)^2}{2}$; $F_\lambda(l)=\lambda F(\lambda^{-1}l)=\min(l-\frac{l^2}{2\lambda},\frac{\lambda}{2})$; $v(\lambda,l)=v(\lambda^{-1}l)=(1-\frac{l}{\lambda})_{(0,\lambda)}$.
\end{enumerate}
 In this example, linear SP-regularizer\cite{Jiang2014Easy} is derivated from the weight function that linearly weights the sample whose loss is between $0$ and $\lambda$.

The second method is called the flvF$\lambda$ method. Its main process for one dimension sub-problem includes the following steps:
\begin{enumerate}
  \item $R_{SP}(v)$ satisfy:  \begin{itemize}
                               \item $dom\ R_{SP}(v) \subset [0,1]^n$;
                               \item $0,1\in cl(dom\ R_{SP}(v))$;
                               \item $R_{SP}(v)$ is convex and continuous;
                             \end{itemize}
  \item $l(v)=\partial(-R_{SP}(v))$;
  \item $v(l)=l^{-1}(v)$;
  \item $F(l)=\langle v(l),l\rangle+R_{SP}(v(l))$;
  \item   $R_{SP}(v,\lambda)=\lambda R_{SP}(v)$; $F_\lambda(l)=\lambda F(\lambda^{-1}l)$; $v(\lambda,l)=v(\lambda^{-1}l)$.
\end{enumerate}

We also present an example for using this method to design SPL.

\begin{enumerate}
  \item $R_{SP}(v)=-\log v \ v\in(0,1]$;
  \item $l(v)=\partial (-R_{SP}(v))=\left\{\begin{aligned}
v^{-1} &  & v\in (0,1), \\
(-\infty,1] &  & v=1;
\end{aligned}
\right.$
  \item $v(l)=\min(1,l^{-1})$;
  \item $F(l)=\langle v(l),l\rangle +R_{SP}(v(l))=\left\{\begin{aligned}
1+\log l, &  & l\in (1,+\infty), \\
l, &  & l\in (-\infty,1];
\end{aligned}
\right.$
  \item \begin{itemize}
          \item $R_{SP}(v,\lambda)=\lambda R_{SP}(v)=-\lambda \log v  \ v\in(0,1]$;
          \item $F_\lambda(l)=\lambda F(\lambda^{-1}l)=\left\{\begin{aligned}
\lambda+\log l-\log \lambda &  & l\in (\lambda,+\infty), \\
l &  & l\in (-\infty,\lambda];
\end{aligned}
\right.$
          \item $v(\lambda,l)=v(\lambda^{-1}l)=\min(1,\lambda l^{-1})$.
        \end{itemize}
\end{enumerate}
 In this example,  the weight function, which weights the sample by the minimal of 1 and $\lambda$ times its loss reciprocal, is derived from the LOG-like SP-regularizer.
\section{Concave conjugate theory for SPCL\label{sec4}}

In the conventional SPCL strategies, a curriculum region needs to be specified and added into a general SPL optimization as a constraint~\cite{Jiang2015Self}. In this way, however, the latent objective of SPL as deduced in the previous sections is changed and cannot be obtained by the previous theory. We thus attempt to discuss this point, and provide explicit latent objective functions underlying SPCL for two specific curriculums. For notation convenience, in the following we omit $\lambda$ in SPL functions.

\subsection{Latent objective of SPCL}
In the following theorem we propose the form of the latent objective underlying SPCL.

\begin{theorem}\label{New Object Function of Curriculum Region}
  Suppose the self-paced regularizer is $R_{SP}(\mathbf{v})$ satisfying the simplified conditions of SPL. Let $F(\mathbf{l})=\inf \limits_{\mathbf{v}\in \mathcal{R}^n} \{\langle \mathbf{v},\mathbf{l} \rangle+R_{SP}(\mathbf{v})\}$ denote the concave conjugate of $-R_{SP}(\mathbf{v})$ in $\mathbf{v}$. $\Psi$ is closed convex set and $ri([0,1]^n\cap \Psi)\neq \emptyset$ and $\delta(\mathbf{v}|\Psi)$ is the indicator function. Then

 $$F^{new} (\mathbf{l})\triangleq \inf \limits_{\mathbf{v}\in \Psi} \{\langle \mathbf{v},\mathbf{l} \rangle+R_{SP}(\mathbf{v})\}=F\oplus \delta^*(\cdot|\Psi)(\mathbf{l}),$$
and
 $$\inf \limits_{f\in\mathcal{F},\mathbf{v}\in [0,1]^n\cap \Psi} E(f,\mathbf{v})=\inf_{f\in \mathcal{F}}\{R_\mathcal{F}(f)+F^{new}(\mathbf{l}(f))\}.$$
\end{theorem}

\begin{proof}
   $$\inf \limits_{\mathbf{v}\in \Psi} \{\langle \mathbf{v},\mathbf{l} \rangle+R_{SP}(\mathbf{v})\}=\inf \limits_{\mathbf{v}\in \mathcal{R}^n} \{\langle \mathbf{v},\mathbf{l} \rangle+R_{SP}(\mathbf{v})-\delta(\mathbf{v}|\Psi)\}$$ $$=(-R_{SP}(\mathbf{v})+\delta(\mathbf{v}|\Psi))^*=F\oplus \delta^*(\cdot|\Psi)(\mathbf{l}).$$
\end{proof}

From the theorem, we know that the latent objective of SPCL under certain curriculum region $\Psi$ is the sup convolution of the original SPL latent objective without this constraint and the support function on it. There are several properties on this new objective $F^{new} (\mathbf{l})$.
\begin{property} If the conditions of the theorem~\ref{New Object Function of Curriculum Region} hold, then $F^{new}(\mathbf{l})$ has the following properties.

  \begin{itemize}
  \item It is upper semi-continuous and concave since it is the concave conjugate.
  \item It is increasing according to the Theorem~\ref{Monotone Conjugate}.
  \item  $F^{new} (\mathbf{l})\ge \max(F(\mathbf{l}), \delta^*(\mathbf{l}|\Psi))$ due to the property of sup convolution and the fact that $\delta^*(\mathbf{0}|\Psi))\ge 0 $.
\end{itemize}
Moreover, if $R_{SP}(\mathbf{v})$ is strictly convex, it yields that

\begin{itemize}
  \item According to Corollary~\ref{Strictly Convexity implies Differentiability},  $F^{new} (\mathbf{l})$ is differentiable.
\end{itemize}

\end{property}

\subsection{Curriculum function}
Through the above discussion, we may find that the curriculum region can be interpreted as a special family of curriculum function.

Suppose we provide the SPL model by adding a curriculum function $R_{CL}(\mathbf{v})$ which is a closed convex function and satisfies $ri(dom \ R_{CL})\cap ri(dom \ R_{SP})\neq \emptyset$. Then the new latent objective function can be obtained by the following:
$$F^{new}(\mathbf{l})=\inf \limits_{\mathbf{v}\in [0,1]^n}\{\langle \mathbf{v},\mathbf{l} \rangle +R_{SP}(\mathbf{v})+R_{CL}(\mathbf{v})\}=F\oplus(-R_{CL})^*(\mathbf{l}).$$
It can be seen that the curriculum properties depends on the conjugate of the curriculum function and the sup convolution step.

Suppose we have $K$ curriculum functions which are proper closed convex functions, and let $R_{CL0}$ denote $R_{SP}$. If they satisfy $\cap_{i=0}^K ri(dom \ R_{CLi})\neq \emptyset$, then according to Property~\ref{Additive Property} the objective function of SPCL is
$$F^{new}(\mathbf{l})=\oplus_{i=0}^K(-R_{CLi})^*(\mathbf{l}).$$

By introducing a new curriculum function $R_{CL}$ into the model, new latent objective is obtained by sup convolution of original object function and conjugate of the curriculum function. The result can be viewed as the action of the new curriculum on the original latentive objective. We call this action \textbf{Curriculum Action} in the follows for convenience.

\subsection{Basic curriculum region}

Consider the following case that the feasible region of $\mathbf{v}$ is $\mathcal{R}^n$ and the SP-regularizer is $0$, and then $$\inf\limits_{\mathbf{v}\in \mathcal{R}^n}\langle \mathbf{v},\mathbf{l} \rangle=\delta(\mathbf{l}|\mathbf{0}),$$ which means that it takes finite value $0$ when the component of $\mathbf{l}$ equals 0 and it takes $-\infty$ on $\mathcal{R}^n\setminus 0$.

For all proper concave function $f(\mathbf{l})$, it holds that $$f(\mathbf{l})\oplus \delta(\mathbf{l}|\mathbf{0})=f(\mathbf{l}).$$ We can then give the following definition related to curriculums:

\begin{definition}[Basic Curriculum Region] For the SPL model $$\inf\limits_{\mathbf{v}\in \mathcal{R}^n}\langle \mathbf{v},\mathbf{l} \rangle +R_{SP}(\mathbf{v}),$$ we call the $dom\ R_{SP}(\cdot)$ the basic curriculum region.
\end{definition}

The commonly discussed SP-regularizers are defined on $[0,1]^n$. Suppose the regularizer $g(v)=-R_{SP}(v)$ is a concave function being differentiable on $[0,1]^n$, and it can be extended to an open set which contains $[0,1]^n$. According to Property~\ref{SoS} the structure of subdifferential, we can obtain $$\partial g(\mathbf{v})=\left\{\begin{array}{cccc}
\nabla g(\textbf{1})+\mathcal{R}^n_{-} & \mathbf{v}=\textbf{1} \\
\nabla g(\textbf{$a^i$})+< b^i_1 e^1,\cdots, b^i_n e^n> & \mathbf{v}=a^i\in V([0,1]^n)\\
\nabla g(\textbf{0})+\mathcal{R}^n_{+} & \mathbf{v}=\textbf{0} \\
\nabla g(\mathbf{v}) +K(\mathbf{v}) & \mathbf{v}\in \partial[0,1]^n/\{V([0,1])^n\}\\
\nabla g(\mathbf{v}) & \mathbf{v} \in (0,1)^n\\
\end{array}
\right.$$where $a^i$ is the vertex of the hypercube $[0,1]^n$ , $b^i_j=\left\{\begin{array}{ccc}
1 & a^i_j=0\\
-1 & a^i_j=1 \\
\end{array}
\right.$, $\langle b^i_1 e^1,$ $\cdots, b^i_n e^n\rangle$ represents the cone generated by $ b^i_1 e^1,\cdots, b^i_n e^n$ with positive coefficients and $V([0,1]^n)$ represents all the vertices of $[0,1]^n$.

By calculating the inverse of set-valued function $\partial g(\mathbf{v})$, the weight set-valued function $\mathbf{v}(\mathbf{l})$ can be obtained.

\subsubsection{Linear Regularizer}
\begin{definition}[Linear Regularizer]\label{lr}
We call $$R_{SP}(\mathbf{v})=-\lambda^T \mathbf{v}$$ linear regularizer for the SPL model
\end{definition}
Once we select the linear regularizer, we can obtain: $$-R_{SP}(\mathbf{v})=\mathbf{\lambda}^T \mathbf{v}$$
$$\partial (-R_{SP})(\mathbf{v})=\left\{\begin{array}{cccc}
\mathbf{\lambda}+\mathcal{R}^n_{-} & \mathbf{v}=\textbf{1}, \\
\mathbf{\lambda}+< b^i_1 e^1,\cdots, b^i_n e^n> & \mathbf{v}=a^i\in V([0,1]^n)),\\
\mathbf{\lambda}+\mathcal{R}^n_{+} & \mathbf{v}=\textbf{0}, \\
\mathbf{\lambda} +K(\mathbf{v}) & \mathbf{v}\in \partial[0,1]^n/\{V([0,1])^n\},\\
\mathbf{\lambda} & \mathbf{v} \in (0,1)^n.\\
\end{array}
\right.$$
According to the Property~\ref{DoS}, we can obtain that$$\partial F(\mathbf{l})=\left\{\begin{array}{cccc}
\mathbf{v}=\textbf{1}  & \mathbf{l}\in \lambda+(\mathcal{R}^n_{-})^\circ \\
\mathbf{v}=a^i             & \mathbf{l}\in\lambda+< b^i_1 e^1,\cdots, b^i_n e^n>^\circ & \mathbf{v}=a^i\in V([0,1]^n))\\
\mathbf{v}=\textbf{0} & \mathbf{l}\in\lambda +(\mathcal{R}^n_{+})^\circ \\
\cdots
\end{array}
\right.
$$
Hence, the domain of $\partial F(\mathbf{\mathbf{l}})=v(\mathbf{l})$ can be separated into $2^n$ part, each taking the same value corresponding to the vertex of the hypercube $[0,1]^n$.

\subsection{Linear homogeneous curriculum}

One of the most commonly used curriculum is the partial order curriculum. For instance, if one has the prior knowledge that example 1 is more important or reliable than example 2, it's reasonable to restrict their feasible region such that $v_1\ge v_2$. In regard to $v_1-v_2\ge0$, we call it linear homogeneous curriculum. Generally, those knowledge come as a series of linear inequalities and we call them partial order curriculum. For simplicity, in the following we consider the simple linear homogeneous curriculum and, for more curriculums, we can treat them one by one.

In order to avoid the disfunctional curriculum and to make analysis convenient, we render the following nonsingular assumption for the curriculum region.

\begin{assumption}[Assumption for Curriculum Region\label{Assumption for Curriculum Region}] A curriculum region $\Psi$ satisfies the following conditions:
 \begin{itemize}
  \item $int(\Psi) \cap int(dom \ R_{SP})\neq \emptyset,$
  \item $\Psi\cap dom\  R_{SP}\neq dom\  R_{SP}.$
\end{itemize}
\end{assumption}

\begin{definition}[Linear Homogeneous Curriculum] If $\Psi=\{\mathbf{v}|\mathbf{v}^T \mathbf{k}\ge0\}$, we call $\Psi$ a linear homogeneous curriculum and $\mathbf{k}$ the linear homogeneous curriculum direction.
\end{definition}

We can then prove the following result:
\begin{theorem}\label{Geometric Effect of Linear Homogeneous Curriculum}
   Suppose $R_{SP}(v)$ satisfies Definition~\ref{assumption for SP-regularizer} and the curriculum as $\mathbf{v}^T \mathbf{k}\ge 0$ corresponding to $\Psi=\{\mathbf{v}|\mathbf{v}^T \mathbf{k}\ge0\}$. If $\Psi$ satisfies Assumption~\ref{Assumption for Curriculum Region}, then we have:
   $$F^{new} (\mathbf{l})=F\oplus \delta(\cdot|\Psi^\circ)(\mathbf{l})=\sup \limits_{\mathbf{l}^1+\mathbf{l}^2=\mathbf{l}}\{F(\mathbf{l}^1))+\delta(\mathbf{l}^2|\Psi^\circ)\}=\sup \limits_{\mathbf{l}^1\in \mathbf{l}-ray_\mathbf{k}} F(\mathbf{l}^1 \rangle\ge 0\}$$ is another non-empty closed convex cone and  $ray_\mathbf{k}$ denotes the ray starting from the origin in direction $\mathbf{k}$.
\end{theorem}
Proof in Appendix~E.

Theorem~\ref{Geometric Effect of Linear Homogeneous Curriculum} illustrates that the latent objective of SPCL is the supremum of the original objective function of SPL without curriculum constraint on the ray which starts from $\mathbf{l}$ to the direction $-\mathbf{k}$.
We then give the theorem on the action of linear homogeneous curriculum.

\begin{theorem}[Action of Linear Homogeneous Curriculum\label{Influence of Linear Homogeneous Curriculum}]
  Suppose $R_{SP}(\mathbf{v})$ is essential strictly convex and satisfies Definition~\ref{assumption for SP-regularizer}. Suppose we have the curriculum constraint $\mathbf{v}^T \mathbf{k}\ge 0$, corresponding to the curriculum region $\Psi=\{\mathbf{v}|\mathbf{v}^T \mathbf{k}\ge0\}$. Then if $\Psi$ satisfies assumption~\ref{Assumption for Curriculum Region},
it holds that $\nabla F^{new}(\mathbf{l})^T\mathbf{k}=v^{new} (\mathbf{l})^T\mathbf{k} \ge 0$ and $$F^{new} (\mathbf{l})=\left \{                 
  \begin{array}{ccc}   
    F(\mathbf{l}) & \mathbf{l}\in \partial (-R_{SP})(\mathbf{k}^\perp)-ray_\mathbf{k}, \\  
    \sup\limits_{\mathbf{l}'\in \mathbf{l}+line_\mathbf{k}}F(\mathbf{l}')(\ge F(\mathbf{l})) & \mathbf{l}\in \partial (-R_{SP})(\mathbf{k}^\perp)+ray_\mathbf{k}. \\  
  \end{array}
\right.$$.
\end{theorem}
The proof is presented in Appendix~F.

Theorem~\ref{Influence of Linear Homogeneous Curriculum} illustrates the form of the latent objective of SPCL following the restriction imposed on the curriculum region. This naturally leads to the following concept of the critical region:
We call $(-R_{SP})(k^\perp)$ the critical region for the new latent objective of SPCL.

According to Theorem~\ref{Influence of Linear Homogeneous Curriculum}, the most important thing for determining $F^{new}(l)$ is to determine the critical region $(-R_{SP})(k^\perp)$, since the critical region divides the $\mathcal{R}^n$ into two parts.

On one part, the linear homogeneous curriculum has no effects. On the other part, the more increase in the curriculum direction, the bigger penalization of the new latent function on the loss.

\subsubsection{Partial Order Curriculum}

We can then evaluate the insights of SPCL with partial order curriculum, where $v_1\ge v_2$ encoding the prior knowledge that example 1 is more important and reliable than example 2. The SP-regularizer is chosen to be the exponential $R_{SP}(v)=v\log v-v+1$.
We can then deduce the following results by using the aforementioned theoretical results:
\begin{itemize}
\item The original SPL latent objective (without this curriculum) is $F(l)=1-e^{-l_1}+1-e^{-l_2}$;
\item $l(v)=\partial(-R_{SP}(v))=-(\log v_1,\log v_2)^T$;
\item Curriculum direction is $k=(1,-1)^T$;
\item The critical region is $line_{(1,1)^T}$;
\item The new latent objective is $$F^{new} (l)=\left \{                 
  \begin{array}{ccc}   
    1-e^{-l_1}+1-e^{-l_2} &l_1\le l_2, \\  
    2(1-e^{-\frac{l_1+l_2}{2}}) & l_1\ge l_2 ;\\  
  \end{array}
\right.$$
\item The new weighting function $$v^{new} (l)=\left \{                 
  \begin{array}{ccc}   
    (e^{-l_1},e^{-l_2})^T &l_1\le l_2, \\  
    (e^{-\frac{l_1+l_2}{2}},e^{-\frac{l_1+l_2}{2}})^T & l_1\ge l_2. \\  
  \end{array}
\right.$$
\end{itemize}

When hard SP-regularizer is chosen, the partial order curriculum can determine the learning order of each example \cite{Jiang2015Self}.

\subsubsection{Linear Curriculum}

It's natural to extend the above discussion to the more general linear curriculum case, as defined in the following.
\begin{definition}[Linear Curriculum] If $\Psi=\{\mathbf{v}|\mathbf{v}^T \mathbf{k}\ge b\}$, we call $\Psi$ the linear curriculum and we call $\mathbf{k}$ the linear curriculum direction.
\end{definition}

Then we can prove the following result:
\begin{theorem}\label{Influence of Linear Curriculum}
  Suppose $R_{SP}(\mathbf{v})$ is essential strictly convex and satisfies Definition~\ref{assumption for SP-regularizer}. Suppose the curriculum is $\mathbf{v}^T \mathbf{k}\ge b$ corresponding to the curriculum region $\Psi=\{\mathbf{v}|\mathbf{v}^T \mathbf{k}\ge b\}$. If $\Psi$ satisfies Assumption~\ref{Assumption for Curriculum Region},
then $$F^{new} (\mathbf{l})=\left \{                 
  \begin{array}{ccc}   
    F(\mathbf{l}) & \mathbf{l}\in \partial g(\mathbf{k}^\perp+\frac{b\mathbf{k}}{\Vert \mathbf{k}\Vert _2})-ray_\mathbf{k}, \\  
    F(\mathbf{l}-\beta^0(\mathbf{l})\mathbf{k})+\beta^0 b(\ge F(\mathbf{l})) & \mathbf{l}\in \partial g(\mathbf{k}^\perp+\frac{b\mathbf{k}}{\Vert \mathbf{k}\Vert _2})+ray_\mathbf{k}, \\  
  \end{array}
\right.$$
where  $\beta^0(\mathbf{l})=\max \arg \limits_\beta \{\nabla F(\mathbf{l}-\beta \mathbf{k})^T \mathbf{k}=0\}$.
\end{theorem}
Proof in Appendix~G.

Theorem~\ref{Influence of Linear Curriculum} helps to obtain the latent objective under SPCL with linear curriculum, which shares the similar structure comparing the one in Theorem~\ref{Influence of Linear Homogeneous Curriculum}. The critical region here becomes $ \partial g(k^\perp+\frac{bk}{\Vert k\Vert _2})$. It is easy to see that the linear curriculum punishes the loss on one side of the critical region and keeps the other side the same.

\subsection{Group Curriculum}
When we have the prior knowledge that some training samples are coming from the similar group with similar importance weights, we can group all samples into multiple categories and make each group share similar weights. In regard to weight sharing, we call it the group curriculum.

Suppose the original $R_{SP}$ regularizer of each example is the same and SPL model can be separated into 1 dimesion sub-problem.
$$\inf \limits_{\mathbf{v}\in [0,1]^n} E(f,\mathbf{v};\lambda)=\inf \limits_{v_1\in [0,1]}\{v_1 l_1+R_{SP}(v_1,\lambda)\}+\cdots+\inf \limits_{v_n\in [0,1]}\{v_n l_n+R_{SP}(\mathbf{v}_n,\lambda)\}+R_\mathcal{F}(f).$$

Suppose $\{i_1,\cdots,i_{s_i}\}$ are members of the group $i$ and there are $k$ groups. By adding the group curriculum, we obtain the new object function as:
$$\inf \limits_{v_1\in [0,1]}\{s_1R_{SP}(v_1,\lambda)+v_1\sum_{j=1}^{s_1} l_{1j}\}+\cdots+\inf \limits_{v_k\in [0,1]}\{s_kR_{SP}(v_k,\lambda)+v_k\sum_{j=1}^{s_k} l_{kj}\}+R_\mathcal{F}(f).$$
Suppose the latent objective function in SPL without such curriculum constraint for one example is $\bar{F_\lambda}$, and then
the new latent objective is
$$F^{new}(\mathbf{l})=(s_1 \star\bar{F_\lambda})(\sum_{j=1}^{s_1} l_{1j})+\cdots+(s_k \star\bar{F_\lambda})(\sum_{j=1}^{s_k} l_{kj}), $$
where $s_1 \star\bar{F_\lambda}(x)\triangleq s_1 \bar{F_\lambda}(s_1^{-1}x)$.

It can be seen that the group curriculum corresponds to a special curriculum region that weights in the same group are restricted to be the same. In regard to knowledge confidence of different group, the partial order curriculum can also be introduced together.

\section{Conclusion}
 In this paper we have established a systematic theory for analyzing the SPL and SPCL via the concave conjugacy theory. We prove that the SPL model corresponds to optimizing a concave, increasing and continuous latent objective function. The relations among weight function, latent objective function and SP-regularizer can thus be obtained explicitly. Furthermore, two general methods for designing the SPL model has been rendered under this theoretical framework. Besides, the latent objective function of SPCL can be derived, instead of only those of SPL by some conventional methods. Such a study tends to be beneficial to facilitate deeper understandings and broader applications of SPL in the future research.

\address{School of Mathematics and Statistics, Xi'an Jiaotong University\\
Xian Ning West Road, Xi'an, Shaan'xi, 710049, P.R. China\\
\email{dymeng@mail.xjtu.edu.cn}}


\end{document}